\documentclass[conference,compsoc]{IEEEtran}

\usepackage{cite}
\usepackage{amsmath,amssymb,amsfonts}
\usepackage{algorithmic}
\usepackage{graphicx}
\usepackage{textcomp}
\usepackage{xcolor}
\usepackage{times}
\usepackage{url}
\usepackage[hidelinks]{hyperref}
\usepackage[utf8]{inputenc}
\usepackage[small]{caption}
\usepackage{graphicx}
\usepackage{subfigure}
\usepackage{booktabs}
\usepackage{algorithm}
\newsavebox{\tablebox}

\newtheorem{defn}{Definition}
\newtheorem{proof}{Proof}
\usepackage[misc]{ifsym}

\newcommand{\sig}[1]{{\small\textsf{{#1}}}}
\newcommand{\qed}{\hfill $\square$}

\newcommand{\Comment}[1]{}

\newtheorem{lemma}{Lemma}

\usepackage{tikz}
\usepackage{textcomp}
\usepackage{hyperref}
\newcommand\copyrighttext{%
  \footnotesize \textcopyright 2021 IEEE. Personal use of this material is permitted.
  Permission from IEEE must be obtained for all other uses, in any current or future
  media, including reprinting/republishing this material for advertising or promotional
  purposes, creating new collective works, for resale or redistribution to servers or
  lists, or reuse of any copyrighted component of this work in other works.
  DOI: 10.1109/AITEST52744.2021.00020}
\newcommand\copyrightnotice{%
\begin{tikzpicture}[remember picture,overlay]
\node[anchor=south,yshift=10pt] at (current page.south) {\fbox{\parbox{\dimexpr\textwidth-\fboxsep-\fboxrule\relax}{\copyrighttext}}};
\end{tikzpicture}%
}

\begin{document}
%
\title{Testing Autonomous Systems with Believed Equivalence Refinement}

\author{\IEEEauthorblockN{Chih-Hong Cheng}
\IEEEauthorblockA{DENSO AUTOMOTIVE Deutschland GmbH\\
Eching, Germany\\
Email: c.cheng@eu.denso.com}
\and
\IEEEauthorblockN{Rongjie Yan}
\IEEEauthorblockA{State Key Laboratory of Computer Science, ISCAS\\
Beijing, China\\
Email: yrj@ios.ac.cn}}

\maketitle
\copyrightnotice

\begin{abstract}
Continuous engineering of autonomous driving functions commonly requires deploying vehicles in road testing to obtain inputs that cause problematic decisions. Although the discovery leads to producing an improved system, it also challenges the foundation of testing using equivalence classes and the associated relative test coverage criterion. In this paper, we propose \emph{believed equivalence}, where the establishment of an equivalence class is initially based on expert belief and is subject to a set of available test cases having a consistent valuation. Upon a newly encountered test case that breaks the consistency, one may need to \emph{refine} the established categorization in order to split the originally believed equivalence into two. 
Finally, we focus on  modules implemented using deep neural networks where every category partitions an input over the real domain. We present both analytical and lazy methods to suggest the refinement. The concept is demonstrated in analyzing multiple autonomous driving modules, indicating the potential of our proposed approach. 
\end{abstract}


%
\IEEEpeerreviewmaketitle

\section{Introduction}

Mainstream approaches for engineering autonomous vehicles follow a paradigm of  \emph{continuous improvement}. It is widely conceived in industry that unconsidered engineering issues such as edge cases may appear when validating the  autonomous vehicles on road with sufficiently many miles. The autonomy function, upon a newly encountered edge case, will be evaluated against the impact and function modification may be performed. While discovering the previously unconsidered issues  leads to an improvement of the engineered system, it also challenges the previously stated assumptions being established in design and in test time.

In this paper, we formally address the impact of newly discovered issues over the previously stated \emph{test-completeness claims}  based on \emph{equivalence classes and test coverage}. Intuitively, equivalence classes are situations where the system demonstrates similar behavior, and test coverage is a \emph{relative completeness} measure to argue that the testing efforts have covered the entire input space with sufficient diversity.  
While the concept of equivalence class may be straightforward in theory, to apply it in complex autonomous driving settings, we mathematically formulate a revision called \emph{believed equivalence class}. The initial categorization used to build equivalence classes is initially based on expert judgements (e.g., all images having sunny weather and straight roads form a believed equivalence class for front-vehicle detection, as experts think that a front vehicle will be detected) and is subject to available test cases
\emph{demonstrating similar behavior}. 
So long as any  newly discovered test cases in the same class violate the behavior similarity (e.g., on sunny weather but the straight road is covered with snow, the perception will fail to detect a white vehicle in the front), there is a need to \emph{refine} the believed equivalence class.  
Regarding the impact of refinement over test coverage, while the number of equivalence classes is known to be exponential (aka \emph{combinatorial explosion}), under the restricted setting of achieving full coverage using $\gamma$-way combinatorial testing~\cite{nie2011survey}, the maximum number of additional test cases required to return to full coverage (due to refinement) is bounded by a polynomial with $\gamma-1$ degree.

Finally, we look into the testing of functions implemented using \emph{deep neural networks (DNNs)}. When the equivalence class is based on partitioning the input domain over real numbers, guiding the construction of refined equivalence classes can utilize the known approach of \emph{adversarial perturbation}~\cite{szegedy2013intriguing,moosavi2017universal}. This avoids creating a believed equivalence class that only contains a single point. For deep neural networks, we prove that finding an optimal cut can be encoded using mixed-integer linear programming (MILP). Due to the scalability considerations of purely analytical methods, we also utilize the concept of $k$-nearest neighbors among the existing test data to guide the perturbed direction, in order to find feasible cuts that at least consistently separate points along all perturbed directions.

To demonstrate the validity of the approach, we have applied the technique on DNN-based modules for trajectory following and adaptive cruise control in autonomous driving. Our  research prototype continuously refines the believed equivalence upon newly encountered test cases that lead to inconsistency.  Altogether our proposal (equivalence classes based on currently available test cases) provides a rigorous framework to understand better the concept of continuous testing and its impact on the completeness of verification of autonomous systems using test coverage. 

The rest of the paper is structured as follows. After  Section~\ref{sec.related.work} comparing our results with existing work, Section~\ref{sec.believed.equivalence} formulates the concept  of believed equivalence, coverage, and the refinement of  believed equivalence. Section~\ref{sec.application.in.dnn} discusses how believed equivalence can be applied when the module under test is a deep neural network. Finally, we summarize our evaluation in Section~\ref{sec.evaluation} and conclude in Section~\ref{sec.concluding.remarks}.

\begin{figure}[t]
\centering
\includegraphics[width=\columnwidth]{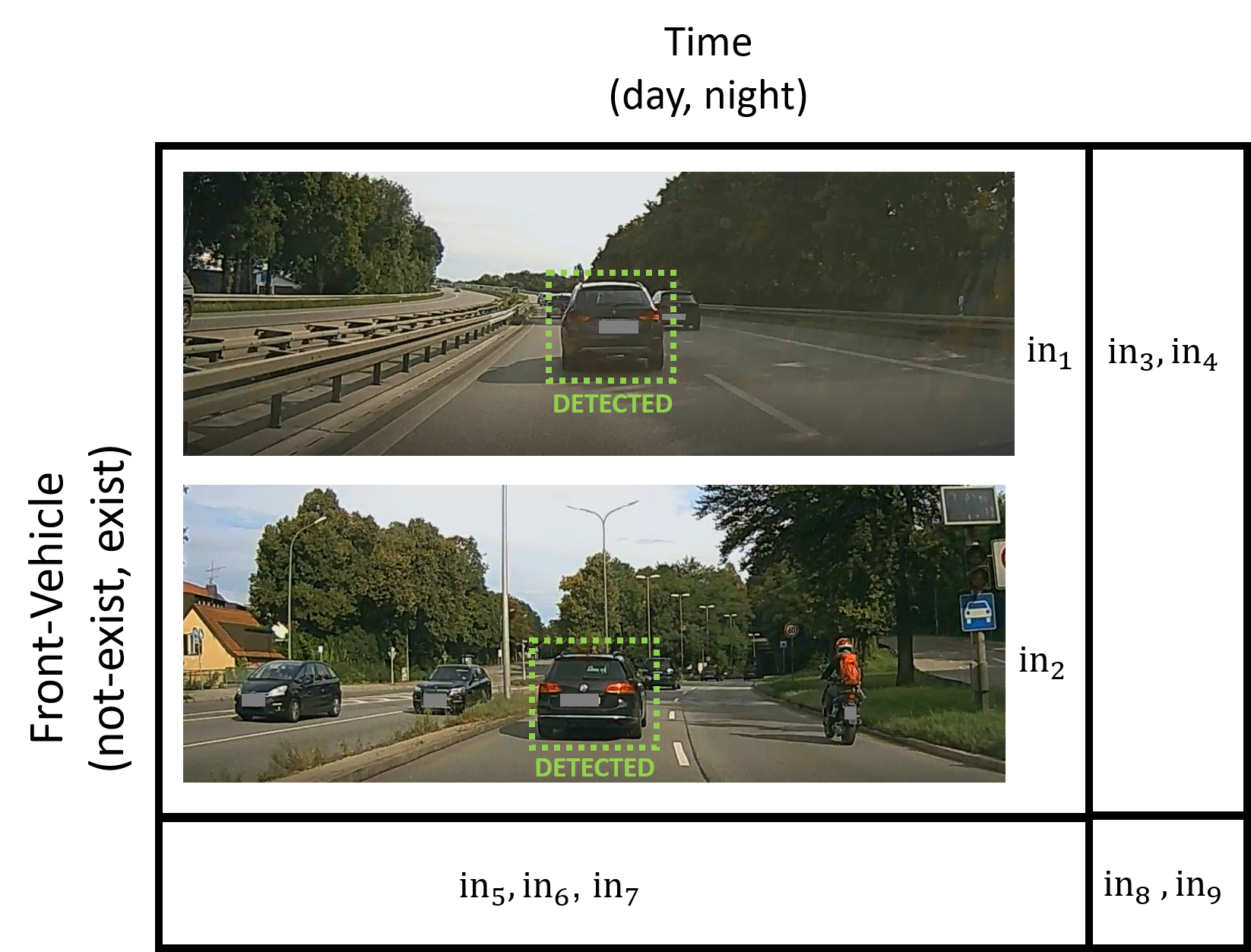}
\caption{Illustrating believed equivalence}
\label{fig.example.categorization}
\end{figure}

\section{Related Work}\label{sec.related.work}

The concept of equivalence class has been widely used in analyzing complex systems. For example, region construction  used in timed automata~\cite{alur1999timed} establishes equivalence classes over infinitely many states regarding the ordering of clock variables reaching integer values. This (finite number of regions) serves as a basis of decidability results for verifying properties of timed automata. State-of-the-art guidelines for safety engineering autonomous driving systems such as ISO TR~4804~\cite{saFAD}\footnote{An early version (aka SaFAD) drafted by multiple automotive manufacturers and suppliers is publicly available at~\url{https://www.daimler.com/documents/innovation/other/safety-first-for-automated-driving.pdf}} admit the need for continuous engineering; ISO TR 4804 also explicitly mentions equivalence class as a method to avoid excessive testing efforts. Nevertheless, it is not mathematically defined and  lacks formal guarantees compared to the classical equivalence class concept. Finally,  there are proposals such as micro-ODD~\cite{koopman2019autonomous} in the safety engineering domain to fine-grain the allowed driving conditions. Still, these proposals  lack a formal treatment in terms of an equivalence class. All the above results provide a strong motivation to our work: Our formulated believed equivalence class only demonstrates equivalent behavior on the currently available test cases and may be subject to refinement in a continuous engineering paradigm. 

For testing DNN-based autonomous systems,  researchers proposed structural coverage metrics to achieve a certain level of ``relative completeness"~\cite{pei2017deepxplore,ma2018deepgauge,sun2018concolic}. Still, recent results have indicated their limited applicability~\cite{li2019structural}. Utilizing scenario description languages~\cite{fremont2019scenic,Msdl}, one can naturally partition the input space for testing~\cite{cheng2018quantitative,LI2020106200}, but there is no guarantee that all inputs within the partition demonstrate similar behavior. Therefore, one may only process a weaker guarantee (equivalence by currently available test cases), as reflected by our proposal.

Finally, for deep neural networks, our construction of refining the believed equivalence integrates the idea of local robustness (see~\cite{huang2017safety,seshia2018formal,dreossi2019formalization} for recent surveys) as used in deep learning. Due to the scalability issues related to exact analytical methods such as MILP, it is realistic to generate a believed equivalence class that contains local robustness region within a reasonable time budget; one can perform refinement upon newly encountered test cases. Our adoption of $k$-nearest neighbors serves for this purpose.

\section{Believed Equivalence  and Refinement}\label{sec.believed.equivalence}

\subsection{Believed Equivalence}\label{subsec.believed.equivalence}

Let $f: \mathcal{D}_{in} \rightarrow \mathcal{D}_{out}$ be the function under analysis, where $\mathcal{D}_{in}$ and $\mathcal{D}_{out}$ are the input and the output domain respectively. Let $evl: \mathcal{D}_{in} \times \mathcal{D}_{out} \rightarrow \mathcal{D}_{evl}$ be an \emph{evaluation function} that assigns an input-output pair with a value in the evaluation domain $\mathcal{D}_{evl}$. 
Let $\mathcal{C} = \{C_1, \ldots, C_m\}$ be the set of \emph{categories}, where for $i \in \{1, \ldots, m\}$, each category $C_i = \{c_{i,1}, \ldots, c_{i,|C_i|}\}$ has a set of \emph{elements} with the set size being~$|C_i|$. In our example, we may use $C_i.c_{i,j}$ to clearly indicate the element and its belonging category. Each element $c_{i,j}: \mathcal{D}_{in} \rightarrow \{\sig{false}, \sig{true}\}$ evaluates any input in the domain to $\sig{true}$ or $\sig{false}$. We assume that for any input $\sig{in} \in \mathcal{D}_{in}$, for every~$i$ there exists exactly one~$j$ ($\exists! j \in \{1, \ldots, |C_i|\}$) such that $c_{i,j}(\sig{in}) = \sig{true}$, i.e., for each category, every input is associated to exactly one element. 

In the following, we define believed equivalence over the test set. Intuitively, for any two data points in the test set, so long as their corresponding element information are the same, demonstrating believed equivalence requires that the results of the evaluation function are also the same. 

\vspace{1mm}

\begin{defn}\label{def.believed.equiv}
Given a set $\mathcal{D}_{test} \subset D_{in}$ of test cases and categories $\mathcal{C}$, function~$f$ \textbf{demonstrates believed equivalence} over~$\mathcal{C}$ under the set $\mathcal{D}_{test}$ and evaluation function~$evl$,  denoted as $f \models^{evl}_{\mathcal{D}_{test}} \mathcal{C}$, if for any two test cases $\sig{in}_1, \sig{in}_2 \in \mathcal{D}_{test}$ the following condition holds:

\begin{equation*}
\resizebox{.9\hsize}{!}{$
\begin{split}
  (\forall i \in \{1, \ldots, m\}, j \in \{1, \ldots, |C_i|\}:  c_{i,j}(\sig{in}_1) = c_{i,j}(\sig{in}_2))  \\
  \rightarrow evl(\sig{in}_1, f(\sig{in}_1)) = evl(\sig{in}_2, f(\sig{in}_2))  
\end{split} $}
\end{equation*}
\end{defn}

Consider the example illustrated in Figure~\ref{fig.example.categorization}. Given $\mathcal{D}_{in}$ to be all possible pixel values in an image from the vehicle's mounted camera, we build a highly simplified input categorization in the autonomous driving perception setup. Let $\mathcal{C} = \{\sig{Time}, \sig{Front-Vehicle}\}$, where  \sig{Time} = $\{\sig{day}, \sig{night}\}$, \sig{Front-Vehicle} = \{\sig{exist}, \sig{not-exist}\} indicates whether in the associated ground truth, there exists a vehicle in the front. Consider~$f$ to be a DNN that performs front-vehicle detection. The first output of~$f$, denoted as $f_1$,  generates~$\sig{1}$ for indicating ``\sig{DETECTED}" (a front vehicle is detected, as shown in Figure~\ref{fig.example.categorization} with dashed rectangles) and~$\sig{0}$ for ``\sig{UNDETECTED}". Based on the input image, let $evl$ be the function that checks if the detection result made by the DNN is consistent with the ground truth, i.e., $evl(\sig{in}, f(\sig{in})) = \sig{true} $ \emph{iff} $(\sig{Front-Vehicle.exist}(\sig{in}) = \sig{true} \leftrightarrow f_1(\sig{in})=\sig{1})$.

As illustrated in Figure~\ref{fig.example.categorization}, let $\mathcal{D}_{test} = \{\sig{in}_{1},\ldots,\sig{in}_{9}\}$. For  $\sig{in}_1$ and $\sig{in}_2$, one can observe from Figure~\ref{fig.example.categorization} that $evl(\sig{in}_1, f(\sig{in}_1)) = evl(\sig{in}_2, f(\sig{in}_2))$. Then $f$ demonstrates believed equivalence over~$\mathcal{C}$ under $\mathcal{D}_{test}$, provided that the following conditions also hold. 

\begin{itemize}
    \item $evl(\sig{in}_3, f(\sig{in}_3)) = evl(\sig{in}_4, f(\sig{in}_4))$,
    \item $evl(\sig{in}_5, f(\sig{in}_5)) = evl(\sig{in}_{6}, f(\sig{in}_{6}))\\ =  evl(\sig{in}_{7}, f(\sig{in}_{7}))$, 
    \item $evl(\sig{in}_8, f(\sig{in}_8)) = evl(\sig{in}_{9}, f(\sig{in}_{9}))$.
\end{itemize}  

\noindent \textbf{(Remark)} It is possible to  control the definition of $evl$ such that demonstrating believed equivalence is independent of the computed value $f(\sig{in})$, thereby creating a partitioning of the input space and its associated properties (i.e., labels). For example, consider the following function. $evl(\sig{in}_1, f(\sig{in}_1)) = (v_1, \ldots, v_m)$ where $v_i = j$ iff $c_{i,j}(\sig{in}) = \sig{true}$.
This evaluation function assigns every input with its associated element index for each category, thereby being independent of~$f$. 

\subsection{Coverage based on Combinatorial Testing}\label{subsec.coverage.combinatorial.testing}

Test coverage can measure the degree that a system is tested against the provided test set. In the following lemma, we associate our formulation of believed equivalence with the \emph{combinatorial explosion} problem of achieving full test coverage. Precisely, the size of the test set to demonstrate believed equivalence while completely  covering all combinations as specified in the categories is  exponential to the number of categories. 

\vspace{1mm}

\begin{lemma}\label{lemma.exponential}
Let $S= \min\{|C_1|, \ldots, |C_m|\}$. Provided that $f \models^{evl}_{\mathcal{D}_{test}} \mathcal{C}$, the minimum size of~$\mathcal{D}_{test}$ to satisfy the below condition is larger or equal to~$S^{m}$.
\begin{equation}\label{eq.exponential}
\begin{split}
\forall  j_1 \in \{1, \ldots, |C_1|\}, \ldots, j_m \in \{1, \ldots, |C_m|\}: \\\exists \sig{in} \in \mathcal{D}_{test}: \bigwedge^{m}_{i=1} c_{i,j_{i}}(\sig{in}) = \sig{true}
\end{split}
\end{equation}

\end{lemma}

\begin{proof}
Recall in Definition~\ref{def.believed.equiv}, $f \models^{evl}_{\mathcal{D}_{test}} \mathcal{C}$ requires any ``two" test cases $\sig{in}_1, \sig{in}_2 \in \mathcal{D}_{test}$ where $\forall i, j: c_{i,j}(\sig{in}_1) = c_{i,j}(\sig{in}_2)$,  $evl(\sig{in}_1, f(\sig{in}_1))$ equals $evl(\sig{in}_2, f(\sig{in}_2))$. Therefore, for Equation~\ref{eq.exponential}, if there exists only one element $\sig{in}\in\mathcal{D}_{test}$ in each element combination $\bigwedge^{m}_{i=1}c_{i,j_{i}}(\sig{in}) = \sig{true}$ (in Eq.~\ref{eq.exponential}), it still satisfies Definition~\ref{def.believed.equiv}. The number of all possible element combinations is $|C_1|\times \ldots \times|C_m|\geq S^{m}$. One needs to cover each element combination with one test case, and a test case satisfying one combination cannot satisfy another combination, due to our previously stated assumption that $\forall i \exists! j \in \{1, \ldots, |C_i|\}: c_{i,j}(\sig{in}) = \sig{true}$. Therefore the lemma holds. \qed
\end{proof}

\vspace{1mm}
As the number of categories can easily achieve~$50$, even when each category has only~$2$ elements, achieving full coverage using the condition in Eq.~\ref{eq.exponential} requires at least $2^{50}$ different test cases; this is in most applications unrealistic. Therefore, we consider asserting full coverage using $\gamma$-way combinatorial testing~\cite{nie2011survey}, which only requires the set of test cases to cover, for every pair ($\gamma=2$) or triple ($\gamma=3$)  of categories, all possible combinations. The following lemma states the following: Provided that~$\gamma$ is a \emph{constant}, the number of test cases to achieve full coverage can be bounded by a polynomial with~$\gamma$ being its highest degree.

\vspace{1mm}
\begin{lemma}[Coverage using $\gamma$-way combinatorial testing]\label{lemma.combinatoral.testing}
Let $S= \max\{|C_1|, \ldots, |C_m|\}$. Provided that $f \models^{evl}_{\mathcal{D}_{test}} \mathcal{C}$, the minimum size of~$\mathcal{D}_{test}$ to satisfy $\gamma$-way combinatorial testing, characterized by the below condition,  is bounded by~${{m}\choose{\gamma}} S^{\gamma}$.
\begin{equation}\label{eq.combinatorial.testing}
\begin{split}
\forall a_1, a_2, \ldots, a_{\gamma} \in \{1, \ldots, m\}: 
\\
\forall  j_{a_{1}} \in \{1, \ldots, |C_{a_{1}}|\}, \ldots, 
j_{a_{\gamma}} \in \{1, \ldots, |C_{a_{\gamma}}|\}: \\
\exists \sig{in} \in \mathcal{D}_{test}: \bigwedge^{\gamma}_{i=1} c_{a_i,j_{a_i}}(\sig{in}) = \sig{true}
\end{split}
\end{equation}

\end{lemma}

\begin{proof}
 To satisfy Equation~\ref{eq.combinatorial.testing}, the first universal term creates ${{m}\choose{\gamma}}$ choices, where for each choice, the second universal term generates at most $S^{\gamma}$ combinations. Therefore, the total number of  combinations induced by two layers of universal quantification equals ${{m}\choose{\gamma}}S^{\gamma}$, and a careful selection of test cases in $\mathcal{D}_{test}$, with each combination covered by exactly one test case, leads to the size of  $\mathcal{D}_{test}$ bounded by ${{m}\choose{\gamma}}S^{\gamma}$.
Finally, similar to Lemma~\ref{lemma.exponential}, covering each combination with only one test case suffices to satisfy $f \models^{evl}_{\mathcal{D}_{test}} \mathcal{C}$. Therefore the lemma holds. \qed
\end{proof}

\begin{figure}[t]
\centering
\includegraphics[width=0.55\columnwidth]{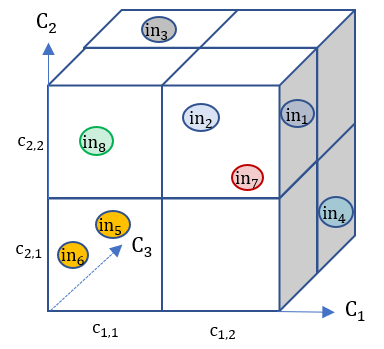}

\caption{Believed equivalence integrated in $2$-way combinatorial testing; the color of each circle indicates the computed value for the evaluation function $evl$.}
\label{fig.combinatorial.testing}
\end{figure}

Figure~\ref{fig.combinatorial.testing} illustrates an example for coverage with $2$-way combinatorial testing, where $C = \{C_1, C_2, C_3\}$ with each $C_i$ ($i \in \{1, 2, 3\}$) having $c_{i,1}$ and $c_{i,2}$ accordingly. Let $\mathcal{D}_{test} = \{\sig{in}_1, \ldots, \sig{in}_6\}$ (i.e., excluding $\sig{in}_7$ and $\sig{in}_8$). Apart from $\mathcal{D}_{test}$ demonstrating believed equivalence,
 $\mathcal{D}_{test}$ also satisfies the condition as specified in Lemma~\ref{lemma.combinatoral.testing}. By arbitrary picking~$2$ categories among $\{C_1, C_2, C_3\}$, for instance $C_1$ and $C_2$, all combinations are covered by elements in $\mathcal{D}_{test}$ ($c_{1,1}c_{2,1} \mapsto \sig{in}_5, \sig{in}_6$; $c_{1,1}c_{1,2} \mapsto \sig{in}_3$; $c_{1,2}c_{2,1} \mapsto \sig{in}_4$;  $c_{1,2}c_{2,2} \mapsto \sig{in}_1, \sig{in}_2$).  

\subsection{Encountering $\mathcal{D}_{in} \setminus \mathcal{D}_{test}$}\label{subsec.refinement}

As the definition of believed equivalence is based on the test set $\mathcal{D}_{test}$, upon new test cases outside the test set, i.e., $\mathcal{D}_{in} \setminus \mathcal{D}_{test}$, the function under analysis may not demonstrate consistent behavior. We first present a formal definition on consistency. 

\vspace{2mm}
\begin{defn}\label{def.consistency}
A test case $\sig{in'} \in \mathcal{D}_{in} \setminus \mathcal{D}_{test}$ is
\textbf{consistent} with the believed equivalence $f \models^{evl}_{\mathcal{D}_{test}} \mathcal{C}$ when the following condition holds for all test cases $\sig{in} \in \mathcal{D}_{test}$:
\begin{equation*}
    (\forall i, j: c_{i,j}(\sig{in}) = c_{i,j}(\sig{in'})) \rightarrow evl(\sig{in}, f(\sig{in})) = evl(\sig{in'}, f(\sig{in'}))
\end{equation*}
Otherwise, $\sig{in'} \in \mathcal{D}_{in} \setminus \mathcal{D}_{test}$ is \textbf{inconsistent} with the believed equivalence $f \models^{evl}_{\mathcal{D}_{test}} \mathcal{C}$.

\end{defn}

\vspace{2mm}
The following result show that, under a newly encountered consistent test case, one can add it to the existing test set while still demonstrating believed equivalence. 

\vspace{2mm}
\begin{lemma}\label{lemma.extension.consistency}
If a test case $\sig{in'} \in \mathcal{D}_{in} \setminus \mathcal{D}_{test}$ is consistent with the believed equivalence $f\models^{evl}_{\mathcal{D}_{test}} \mathcal{C}$, then $f \models^{evl}_{\mathcal{D}_{test} \cup\{\sig{in'}\}} \mathcal{C}$. 
\end{lemma}

\begin{proof}
A direct implication from Definition~\ref{def.believed.equiv} and~\ref{def.consistency}. \qed
\end{proof}

\vspace{2mm}
Consider again the example in Figure~\ref{fig.combinatorial.testing}. Provided that $\mathcal{D}_{test} = \{\sig{in}_1, \ldots, \sig{in}_6\}$, one can observe that $\sig{in}_8$ is consistent with $f \models^{evl}_{\mathcal{D}_{test}} \mathcal{C}$ (as for $c_{1,1}c_{2,2}c_{3,1}$, there exists no test case from $\mathcal{D}_{test}$), while $\sig{in}_7$ is inconsistent (the evaluated result for~$\sig{in}_7$ is not the same as $\sig{in}_1$ and $\sig{in}_2$).

Encountering a new test case $\sig{in'} \in \mathcal{D}_{in} \setminus \mathcal{D}_{test}$ that generates inconsistency, there are two methods to resolve the inconsistency depending on whether~$f$ is implemented correctly.

\begin{itemize}
\item \textbf{(When $f$ is wrong - function modification)} Function modification refers to the change of $f$ to $f'$ such that the evaluation over $\sig{in'}$ turns consistent; it is used when $f$ is engineered wrongly. For the example illustrated in Figure~\ref{fig.combinatorial.testing}, a modification of~$f$ shall ensure that $\sig{in}_1$, $\sig{in}_2$, and $\sig{in}_7$ have the same evaluated value (demonstrated using the same color). 

\item \textbf{(When $f$ is correct - refining the believed equivalence class)} When~$f$ is implemented correctly, it implies that our believed equivalence, based on the currently defined category~$\mathcal{C}$,  is too coarse. Therefore, it is necessary to refine~$\mathcal{C}$ such that the created believed equivalence separates the newly introduced test case from others. To ease understanding, consider again the example in Figure~\ref{fig.example.categorization}. When one encounters an image where the front vehicle is a white trailer truck, and under heavy snows, we envision that the front vehicle can be undetected. Therefore, the currently maintained categories shall be further modified to 

\begin{itemize}
\item differentiate images by snowy conditions (preferably via proposing a new category), and 
\item differentiate images having front vehicles in white from front vehicles having non-white colors (preferably via refining an existing element in a  category). 
\end{itemize}

\end{itemize}

In the next section, we mathematically define the concept of \emph{refinement} to cope with the above two types of changes in categories.

\subsection{Believed Equivalence Refinement}\label{subsec.coverage.decrease.under.refinement}

\begin{defn}\label{def.refinement.cut}
Given~$\mathcal{C}$ and $\mathcal{D}_{test}$, a \textbf{refinement-by-cut}~$RC(\mathcal{C}, i, j)$ generates a new set of categories by first removing $c_{i,j}$ in~ $C_i$, followed by adding two elements $c_{i,j,1}$ and $c_{i,j,2}$ to~$C_{i}$, where the following conditions hold.
\begin{itemize}
    \item For any test case $\sig{in} \in \mathcal{D}_{test}$ such that $c_{i,j}(\sig{in}) =  \sig{true}$, exclusively either $c_{i,j,1}(\sig{in}) =  \sig{true}$ or $c_{i,j,2}(\sig{in}) =  \sig{true}$.
    \item For any test case $\sig{in} \in \mathcal{D}_{test}$ such that $c_{i,j}(\sig{in}) =  \sig{false}$,  $c_{i,j,1}(\sig{in}) =  \sig{false}$ and $c_{i,j,2}(\sig{in}) =  \sig{false}$.    
    
\end{itemize}

\end{defn}

To ease understanding, Figure~\ref{fig.cuts} shows the visualization for two such cuts.\footnote{Note that Figure~\ref{fig.cuts}, the cut using a plane is only for illustration purposes; the condition as specified in Definition~\ref{def.refinement.cut} allows to form cuts in arbitrary shapes.} The cut in the left  refines $c_{2,2}$ to $c_{2,2,1}$ and $c_{2,2,2}$, and the cut in the right refines $c_{1,2}$ to $c_{1,2,1}$ and $c_{1,2,2}$. One can observe that the left cut can enable a refined believed equivalence, as the cut has separated $\sig{in}_7$ from $\sig{in}_1$ and $\sig{in}_2$.  This brings us to the definition of \emph{effective refinement cuts}.

\begin{defn}\label{def.effective.refinement.cut}
Given~$f$, $\mathcal{C} = \{C_1, \ldots, C_m\}$, $\mathcal{D}_{test}$ where $f \models^{evl}_{\mathcal{D}_{test}} \mathcal{C}$, and a new test case $\sig{in'} \in \mathcal{D}_{in} \setminus \mathcal{D}_{test}$,  a refinement-by-cut $RC(\mathcal{C}, i, j)$ is \textbf{effective} over 
$\sig{in'}$, if $f \models^{evl}_{\mathcal{D}_{test} \cup \{\sig{in'}\}} RC(\mathcal{C}, i, j)$.
\end{defn}

\begin{figure}[t]
\centering
\includegraphics[width=\columnwidth]{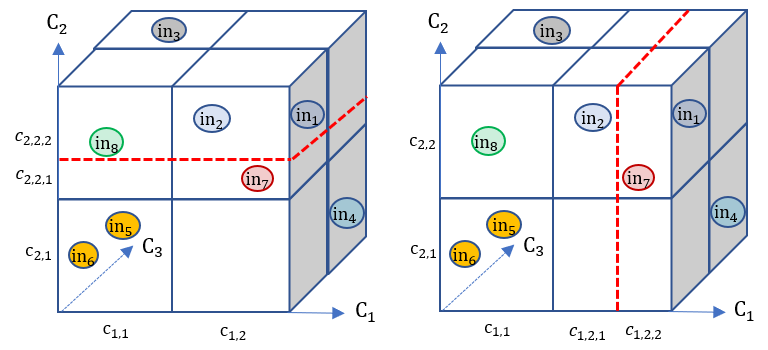}

\caption{Effective (left) and ineffective (right) cuts}
\label{fig.cuts}
   \vspace{-2mm}
\end{figure}

Apart from cuts, the second type of refinement is to add  another category. 

\begin{defn}\label{def.refinement.expansion}
Given~$\mathcal{C} = \{C_1, \ldots, C_m\}$, $\mathcal{D}_{test}$ and a new test case $\sig{in'} \in \mathcal{D}_{in} \setminus \mathcal{D}_{test}$, an \textbf{refinement-by-expansion}~$RE(\mathcal{C})$ generates a new set of categories by adding $C_{m+1}$ to $\mathcal{C}$, where $ C_{m+1} = \{c_{m+1,1},c_{m+1,2}\}$  with elements satisfying the following conditions.
\begin{itemize}
    \item For any test case $\sig{in} \in \mathcal{D}_{test}$, $c_{m+1,1}(\sig{in}) =  \sig{true}$ and $c_{m+1,2}(\sig{in}) =  \sig{false}$. 
    \item  $c_{m+1,1}(\sig{in'}) =  \sig{false}$ and $c_{m+1,2}(\sig{in'}) =  \sig{true}$. 
    
\end{itemize}

\end{defn}

In contrast to refinement cuts where effectiveness is subject to how cuts are defined, the following lemma paraphrases that  refinement-by-expansion is by-definition effective. 

\begin{lemma}\label{def.effective.refinement.expansion}
Given~$f$, $\mathcal{C} = \{C_1, \ldots, C_m\}$, $\mathcal{D}_{test}$ where $f \models^{evl}_{\mathcal{D}_{test}} \mathcal{C}$, and a new test case $\sig{in'} \in \mathcal{D}_{in} \setminus \mathcal{D}_{test}$,  a refinement-by-expansion $RE(\mathcal{C})$ is \textbf{effective} over $\sig{in'}$, i.e., $f \models^{evl}_{\mathcal{D}_{test} \cup \{\sig{in'}\}} RE(\mathcal{C})$.
\end{lemma}

\begin{proof}

\begin{enumerate}
    \item By introducing the new category $C_{m+1}$ following the condition as specified in Definition~\ref{def.effective.refinement.expansion}, consider $\sig{in}_1, \sig{in}_2 \in \mathcal{D}_{test}$ where $\forall i \in \{1, \ldots,m\}, j \in \{1, \ldots, |C_i|\}: c_{i,j}(\sig{in}_1) = c_{i,j}(\sig{in}_2)$, then as $c_{m+1,1}(\sig{in}_1) = c_{m+1,1}(\sig{in}_2)= \sig{true}$ and $c_{m+1,2}(\sig{in}_1) = c_{m+1,2}(\sig{in}_2)= \sig{false}$. We can expand the first universal quantifier such that  $\forall i \in \{1, \ldots,m, m+1\}, j \in \{1, \ldots, |C_i|\}: c_{i,j}(\sig{in}_1) = c_{i,j}(\sig{in}_2)$ also holds. As $evl(\sig{in}_1, f(\sig{in}_1)) = evl(\sig{in}_2, f(\sig{in}_2))$, so $f \models^{evl}_{\mathcal{D}_{test}} RE(\mathcal{C})$. 
    
    \item For any test case $\sig{in} \in \mathcal{D}_{test}$, we have $c_{m+1,1}(\sig{in}) \neq c_{m+1,1}(\sig{in'})$. Therefore, the inconsistency condition as specified in Definition~\ref{def.consistency} does not hold, i.e., $\sig{in'}$ is consistent with $f \models^{evl}_{\mathcal{D}_{test}} RE(\mathcal{C})$. By applying Lemma~\ref{lemma.extension.consistency} we have $f \models^{evl}_{\mathcal{D}_{test} \cup \{\sig{in'}\}} RE(\mathcal{C})$. \qed
\end{enumerate}
\end{proof}

\noindent\textbf{(Coverage under Refinement)} Although introducing refinement over categories can make $\sig{in'}$ consistent with $\mathcal{D}_{test}$, it may also make the previously stated $100\%$ coverage claim (using $\gamma$-way combinatorial testing) no longer valid. An example can be found in the left of Figure~\ref{fig.cuts}, where a full coverage $2$-way combinatorial testing no longer holds due to the cut. By analyzing  category pairs $C_2$ and $C_3$, one realizes that achieving $100\%$ coverage requires an additional test case~$\hat{\sig{in}}$ where $c_{2,2,1}(\hat{\sig{in}}) = \sig{true}$ and $ c_{3,2}(\hat{\sig{in}}) = \sig{true}$.

For $\gamma$-way combinatorial testing, in the worst case, the newly introduced element ($c_{m+1,2}$ in the refinement-by-expansion case; $c_{i,j,1}$ or $c_{i,j,2}$ in the refinement-by-cut case) will only be covered by the newly encountered test case $\sig{in'} \not\in \mathcal{D}_{test}$. As one element has been fixed, the number of additional test cases to be introduced to recover full coverage can be conservatively polynomially bounded by ${{m}\choose{\gamma -1}} S^{\gamma}$, where $S= \max\{|C_1|, \ldots, |C_m|\}$. Under the case of $\gamma$ being a constant, the polynomial is one degree lower than the one as constrained in Lemma~\ref{lemma.combinatoral.testing}.

Finally, we omit technical formulations, but as refinement is operated based on the set of categories, it is possible to define the hierarchy of coverage accordingly utilizing the refinement concept. As a simple example, consider the element $\sig{vehicle}$ being refined to $\sig{vehicle}_{\sig{car}}$, $\sig{vehicle}_{\sig{truck}}$ and $\sig{vehicle}_{\sig{others}}$ to reflect different types of vehicle and their different performance profile. The coverage on $\sig{vehicle}$ can be an aggregation over the computed coverage from three refined elements.

\section{Refinement-by-cut for Inputs over Reals}\label{sec.application.in.dnn}

\begin{figure}[t]
\centering
\includegraphics[width=\columnwidth]{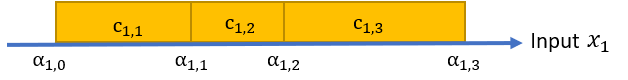}
\caption{Defining a category $C_1 = \{c_{1,1}, c_{1,2}, c_{1,3}\}$ using four values $\alpha_{1,0},\alpha_{1,1}, \alpha_{1,2}, \alpha_{1,3}$, acting on input variable $x_1$}
\label{fig.partition}
\end{figure}

\begin{figure}
\centering
\includegraphics[width=0.8\columnwidth]{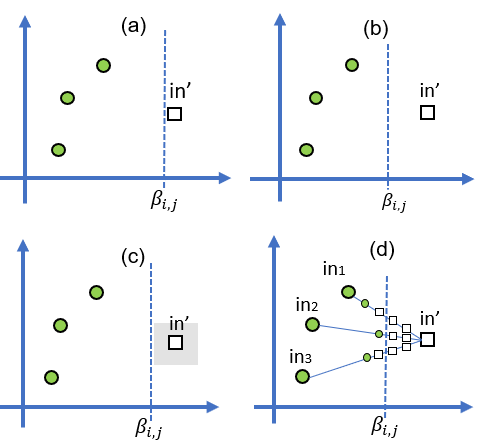}
 
\caption{Optimizing from an arbitrary effective cut (a) to a cut being distant to all points (b), to a cut with local robustness guarantees (c), to  heuristics to check value consistency alone the direction to $3$-nearest neighbors (d)}
\label{fig.optimization}
\end{figure}

\subsection{Refinement-by-cut as Optimization}

Finally, we address the special case where the refinement can be computed \emph{analytically}. Precisely, we consider the input domain $\mathcal{D}_{in}$ being $\mathbb{R}^{m}$, i.e., the system takes~$m$ inputs $x_{1}, \ldots, x_m$ over reals, where $m$ matches the size of the category set. We also consider that every category $C_{i} = \{c_{i,1}, \ldots, c_{i,i_k}\}$ is characterized by $k+1$ ascending values $\alpha_{i,0}, \alpha_{i,1}, \ldots, \alpha_{i,i_k}$, such that $C_i$ performs a \textbf{simple partitioning} over the $i$-th input, i.e., for $\sig{in} = (x_1, \ldots, x_{m})$, $c_{i,j}(\sig{in}) = \sig{true}$ \emph{iff} $x_{i} \in (\alpha_{i, j-1}, \alpha_{i, j}]$. Figure~\ref{fig.partition} illustrates the idea. The above formulation fits well for analyzing systems such as deep neural networks where each input is commonly normalized to a value within the interval $[-1, 1]$.

Using Definition~\ref{def.refinement.cut}, a refinement-by-cut $RC(\mathcal{C}, i, j)$ implies to find a value $\beta_{i,j} \in  (\alpha_{i, j-1}, \alpha_{i, j}]$ such that (1)
$c_{i,j,1} = \sig{true}$ \emph{iff} $x_i \in (\alpha_{i, j-1}, \beta_{i,j}]$ and (2) $c_{i,j,2} = \sig{true}$ \emph{iff} $x_i \in (\beta_{i,j}, \alpha_{i, j}]$. Provided that $\sig{in'} = (x_1', \ldots, x_m')$ being the newly encountered test case, we present one possible analytical approach for finding such a cut, via solving the following optimization problem:

\begin{equation}\label{eq.ana.obj}
    \sig{maximize} \;\; \epsilon 
\end{equation}
subject to:

     \vspace{-3mm}
     
\begin{equation}[\text{Effective}]\label{eq.ana.effective}
\begin{array}{l}
\;\;\;\; f \models^{evl}_{\mathcal{D}_{test} \cup \{\sig{in'}\}} RC(\mathcal{C}, i, j)
\end{array}
\end{equation}

     \vspace{-3mm}

\begin{equation}[\text{Distant}]\label{eq.ana.epsilon}
\begin{array}{l}
    \forall \sig{in} = (x_1, \ldots, x_m) \in \mathcal{D}_{test} \cup \{\sig{in'}\}: \\  \;\;\;\; c_{i,j}(\sig{in}) = \sig{true} \rightarrow  |\beta_{i,j} - x_i| \geq \eta
\end{array}
\end{equation}

     \vspace{-3mm}

\begin{equation}[\text{Robust}]\label{eq.ana.robustness}
\begin{array}{l}
    \forall \sig{in}= (x_1, \ldots, x_m) \in \mathbb{R}^m:\\
  \;\;\;(\forall i \in \{1, \ldots, m\}:
   |x_i - x_i'| \leq \epsilon) \\ 
    \;\;\; \;\;\; \;\;\;\rightarrow  evl(\sig{in}, f(\sig{in})) = evl(\sig{in'}, f(\sig{in'}))
\end{array}
\end{equation}

In the above stated optimization problem, apart from constraints as specified in Equation~\ref{eq.ana.effective} which follows the standard definition as regulated in Definition~\ref{def.effective.refinement.cut} (illustrated in Figure~\ref{fig.optimization}a), we additionally impose two \emph{optional} constraints Eq.~\ref{eq.ana.epsilon} and Eq.~\ref{eq.ana.robustness}.
\begin{itemize}
    \item First, we wish to regulate the size of the newly created element such that the new believed equivalence class generalizes better by keeping a minimal distance $\epsilon$ to the boundary $\beta_{i,j}$ (illustrated in Figure~\ref{fig.optimization}b).
    
    \item  Secondly, we also specify that all inputs within the $\epsilon$-ball centered by $\sig{in'}$ demonstrate consistent behavior, thereby demonstrating robustness via behavioral equivalence  (illustrated in Figure~\ref{fig.optimization}c). 
\end{itemize}

The following lemma states the condition where the optimization problem can be solved analytically using integer linear programming. For general cases, nonlinear programming is required. 

\begin{lemma}\label{lemma.milp.encoding}
Let $f$ be a ReLU network and $evl$ be a linear function over input and output of $f$. Provided that  $C_i$ performs a simple partitioning over the $i$-th input, and every input $x_i$ is bounded by $(\alpha_{i,0}, \alpha_{i,i_k}]$ ($\alpha_{i,0}\neq -\infty$ and $\alpha_{i,i_k} \neq \infty$), the optimization problem as stated using Equation~\ref{eq.ana.obj},~\ref{eq.ana.effective},~\ref{eq.ana.epsilon}, and~\ref{eq.ana.robustness} can be encoded using mixed-integer linear programming (MILP). 
\end{lemma}

\begin{proof} Finding the maximum $\epsilon$ for constraint systems using Equation~\ref{eq.ana.obj},~\ref{eq.ana.effective},~\ref{eq.ana.epsilon} can be trivially by finding the separation plane (illustrated in Figure~\ref{fig.optimization}b). Therefore, the key is to independently compute $\epsilon$ over constraints as specified in Equation~\ref{eq.ana.robustness}, and the resulting $\epsilon$ is the minimum of the computed two. If $f$ is a ReLU network with all its inputs being bounded by closed intervals, it is well known that the $f$ as well as the local robustness condition can be encoded using MILP~\cite{cheng2017maximum,fischetti2018deep,anderson2020strong} due to the piecewise linear nature of ReLU nodes. As long as $evl()$ is a linear function, it remains valid to perform encoding using MILP. \qed
\end{proof}

\noindent\textbf{(Extensions)} Here we omit technical details, but it is straightforward to extend the above constraint formulation to (1) integrate multiple cuts within the same element $c_{i,j}$ and to (2) allow multiple cuts over multiple input variables. 

\subsection{Using $k$-nearest Neighbors to Generate Robust Refinement Proposals}~\label{subsec.k.nearest.neighbors}

     \vspace{-3mm}

Finally, as the local robustness condition in Equation~\ref{eq.ana.robustness} is computationally hard to be estimated for state-of-the-art neural networks, we consider an alternative proposal that checks robustness using the direction as suggested by $k$-nearest neighbors. The concept is illustrated using Figure~\ref{fig.optimization}d, where one needs to find a cut when encountering  $\sig{in'}$. The method proceeds by searching for its closest neighbors  $\sig{in}_1$, $\sig{in}_2$ and $\sig{in}_3$ (i.e., $k=3$), and builds three vectors pointing towards these neighbors. Then perturb $\sig{in'}$ alongside the direction as suggested by these vectors using a fixed step size $\Delta$, until the evaluated value is no longer consistent (in Figure~\ref{fig.optimization}d; all rectangular points have consistent evaluated value like $\sig{in'}$). Finally, decide the cut ($\beta_{i,j}$) by taking the minimum distance as reflected by perturbations using these three directions.  

The generated cut may further subject to refinement  (i.e., newly encountered test cases may lead to further inconsistency), but it can be viewed as a \emph{lazy approach} where (1) we take only finite directions (in contract to local robustness that takes all directions) and (2) we take a fixed step size (in contract to local robustness that has an arbitrary step size). By doing so, the granularity of refinement  shrinks on-the-fly when encountering new-yet-inconsistent test cases.

\begin{algorithm}[t]
\small
\caption{Building an initial believed equivalence class with lazy methods}\label{alg:cut}
\begin{algorithmic}[1]
\STATE{Input: (1) DNN model $f$ taking input $x_1, \ldots, x_m$, with each input variable $x_i$ bounded by $(l_{x_i}, u_{x_i}]$, (2) the collected data set $\mathcal{D}_{test}$, (3) evaluation function $evl()$, and (4) $\epsilon$ for the minimum required distance.}
\STATE{Output: Categorization ${\mathcal{C}}$ that guarantees to maintain believe equivalence $f \models^{evl}_{\mathcal{D}_{test}} \mathcal{C}$, or return \textsf{warning} when it is impossible to find a cut with interval size larger than $\eta$}
\STATE$\mathcal{C}:=\textrm{initializeCategoryUsingVariableBound}(x_1, \ldots, x_m)$
\STATE{\textbf{let} $\mathcal{D}_{exp}  := \emptyset$}
\FOR{$\textsf{in} \in \mathcal{D}_{test}$}
    \FOR{$\textsf{in'}\in \mathcal{D}_{exp}$}
        \IF{($\forall i, j: c_{i,j}(\textsf{in}) = c_{i,j}(\textsf{in}'))   \wedge   (evl(\textsf{in}, f(\textsf{in})) \neq evl(\textsf{in}', f(\textsf{in}')))$}
            \STATE{$\mathcal{D}_{neighbor}=\textrm{getKNearestNeighbor}(\textsf{in},\mathcal{D}_{exp},f)$}
            \STATE{\textbf{let} replaced := \sig{false}}
            \FOR{$i \in \{1, \ldots, m\}$}
                \STATE{$C_i':=\textrm{cutByRefinement}(\textsf{in}, C_i, \mathcal{D}_{neighbor}, f)$}
                \IF{($\textrm{isRefinementIntervalLarger}(C_i, C_i', \eta))$}
                
                    \STATE{$C_i := C_i'$}
                 \STATE{replaced := \sig{true}}
                    \STATE{\textbf{break}}
                \ENDIF
            \ENDFOR
                \IF{replaced $= \textsf{false}$}
                    \STATE{\textbf{return} \textsf{warning}}
                \ENDIF

        \ENDIF
    \ENDFOR
    \STATE{$\mathcal{D}_{exp}=\mathcal{D}_{exp}\cup \{\sig{in}\}$}
\ENDFOR
\STATE{return ${\mathcal{C}}$}
\end{algorithmic}
\end{algorithm}

Integrating the above mindset, we present in Algorithm~\ref{alg:cut} lazy cut refinement for constructing believed equivalence from scratch. Line~3 builds an initial categorization for each variable by using the corresponding bound. We use the set $\mathcal{D}_{exp}$ to record all data points that have been considered in building the believed equivalence, where $\mathcal{D}_{exp}$ is initially set to empty (line~4). For every new data $\sig{in}$ to be considered, it is compared against all data points in $\mathcal{D}_{exp}$; if exists $\sig{in'} \in \mathcal{D}_{exp}$ such that the believed equivalence is violated (line~7), then find the nearest neighbors in $D_{exp}$ (line~8) and pick one dimension to perform cut by refinement (line~8 mimicking the process in Figure~\ref{fig.optimization}(d)). Whenever the generated cut proposal $C_i'$ satisfies the distant constraint as specified in Equation~\ref{eq.ana.epsilon} (line~12), the original $C_i$ is replaced by $C_i'$. Otherwise (lines 18-20), every proposed cut can not satisfy the distant constraint as specified in Equation~\ref{eq.ana.epsilon} so the algorithm raises a warning (line~19). Finally, the data point $\sig{in}$ is inserted to $\mathcal{D}_{exp}$, representing that it has been considered. The presentation of the algorithm is simplified to ease understanding; in implementation one can employ  multiple optimization methods (e.g., use memory hashing to avoid running the inner loop at line~6 completely) to speed up the construction process.

\section{Implementation and Evaluation}\label{sec.evaluation}

We have implemented the concept of believed equivalence and the associated refinement pipeline using Python, focusing on DNN-based modules. The tool searches for feasible cuts on new test cases that lead to inconsistency. As the goal is to demonstrate full automation of the technique, we applied the description in  Section~\ref{sec.application.in.dnn} such that each category matches exactly one input variable. Therefore, cuts can be applied to the domain of any input variable. In our implementation, we maintain a fixed ordering over categories/variables to be cut  whenever refinement is triggered. 
For the evaluation in automated driving using public examples, we use the lane keeping assist (LKA) module\footnote{Details of the example is  available at \url{https://www.mathworks.com/help/reinforcement-learning/ug/imitate-mpc-controller-for-lane-keeping-assist.html}} and the adaptive cruise control (ACC) module\footnote{Details of the example is  available at \url{https://ch.mathworks.com/help/mpc/ug/adaptive-cruise-control-using-model-predictive-controller.html}} made publically available by Mathworks to demonstrate the process of refinement on its believed equivalences.

The LKA subsystem takes six inputs, including lateral velocity (LV),
yaw angle rate (YAR), lateral deviation (LD),  relative yaw angle (RYA), previous steering angle (PSA), and  measured disturbance (MD) to decide the steering angle of the vehicle. We construct a DNN for the implementation of the lane keeping assist. The DNN contains three fully connected hidden layers and adopts ReLU as activation functions.  The model is trained with the given dataset provided by Mathworks, and the accuracy reaches $99.8\%$. The output of the DNN ranges from $-1.04$ to $1.04$, meaning that the steering angle should be between the degree of $[-60, 60]$. We divide the range of its output into five classes in our evaluation function, indicating that the output (steering angle) can be one of the following classes: \textit{strong left, left turn, straight, right turn, or strong right}. Initially, each category has one element being set as the range of the corresponding input. The categories are gradually refined using cuts when inconsistent test cases appear.

\begin{table}[tp]
    \centering\small
    \begin{tabular}{r|c c c c c c}\hline
      $\#$ test cases  & LV & YAR & LD & RYA& PSA & MD  \\\hline

       0  & 1 & 1 & 1 & 1 & 1 & 1\\ 
       1000  & 28 & 14 & 13 & 5 & 2 &1\\ 
       5000 & 29 & 16 & 13 & 10 & 4 & 2\\
       10000 & 31 & 17 & 13 & 11 & 6 &2\\
       20000 & 32 & 19 & 14 & 12 & 9 &3\\
       30000 & 35 & 20 & 15 & 13 & 10 &3\\
       40000 & 39 & 21 & 18 & 13 & 12 &3\\
       50000 & 40 & 21 & 18 & 13 & 14 &3\\
       60000 & 41 & 22 & 18 & 14 &14 &3\\
       70000& 41 & 22 & 18 &15 & 15 &3\\
       80000 & 41 & 22  &18 & 15 & 16 &3\\
       \hline
    \end{tabular}
    \caption{Refinement on believed equivalence on the LKA module}
    \label{tab.cuts}
     \vspace{-5mm}
\end{table}

Table~\ref{tab.cuts} provides a trend summary on the number of cuts applied in each category, subject to the number of test cases being encountered. After the first~$1000$ test cases, LV has been refined to~$28$ intervals, while YAR has been refined to~$14$ intervals. The searching sequence for feasible cuts is from the first category to the last. 
From the table, we observe that after the number of explored test cases exceeds~$40000$, the number of cuts increases slowly. The reason is that the test cases and their evaluated values diverse in the beginning, so the categories are frequently updated. However, with the increasing number of explored test cases, most of the new test cases are consistent with existing ones. 

Note that the number of cuts in every category is affected by the ordering how cuts are applied.  As the search always starts from the first category (LV), the number of cuts for LV is larger than those of the others. 
When we change the searching sequence and start from the last category to the first one (i.e., to start the cut from MD), the number of cuts after the first~$1000$ test cases becomes~$1$, $3$, $9$, $12$, $16$ and~$3$, respectively. 
Though the number of cuts varies with different searching sequence, the number of cuts for the last category (MD) is more stable. The reason being that the number of different valuations of MD in the dataset is small.  

The ACC module takes five inputs, including driver-set velocity ($V_{set}$), time gap ($T_{gap}$), velocity of the ego car ($V_{ego}$), relative distance to the lead car ($D_{rel}$), and relative velocity to the lead car ($V_{rel}$) to decide the acceleration of the ego car. The first two inputs are assigned with fixed valuations.
By setting the range of initial velocities and positions for lead and ego cars, we obtain a set of inputs and outputs from the ACC module using model predictive control
lateral velocity (LV). We then construct a DNN for the implementation of the module. The structure of the DNN is similar to that of LKA. The model is trained with the obtained set, and the accuracy reaches 98.0\%. The output of the DNN ranges from~$-3$ to 2, for the physical limitation of the vehicle dynamics. We also divide the range of its output into five classes, indicating the velocity can be strong braking, slow braking, braking, slow accelerating, or strong accelerating. The categories of inputs are gradually refined when inconsistent test cases appear.  

\begin{table}[tp]
    \centering\small
    \begin{tabular}{r|c c  c}\hline
      $\#$ test cases  & $V_{ego}$ & $D_{rel}$ & $V_{rel}$  \\\hline
      0 & 1 & 1 &1 \\
      1000 &221 & 27 & 1\\
      2000 & 230 & 34 & 1\\
      3000 & 261 & 43 & 3 \\
      4000 & 262 & 53 & 3 \\
      5000 & 284 & 55 & 4 \\
      6000 & 285 & 57 & 4 \\
      7000 & 289 & 58 & 5\\
      8000 & 290 & 62 & 5 \\
      9000 & 305 & 67 & 5\\
      10000 & 307 & 70 & 5\\\hline
    \end{tabular}
    \caption{Refinement on believed equivalence on the ACC module}
    \label{tab.acc}
         \vspace{-5mm}
\end{table}

Table~\ref{tab.acc} lists the trend summary on the number of cuts applied in each category of three inputs (except for the first two), subject to the number of encountered test cases. The searching sequence of feasible cuts is from the first to the last. From the table, we observe that the number of cuts for the velocity of ego increases rapidly when the number of explored test cases is less than 1000. However, the number of cuts for the last input increases slowly. The reason is that the velocity of the ego and the relative distance between the ego and lead car may be the main factors to decide the acceleration valuation of the ego car.  

The refined categories show that  the granularity for most of the cuts  is less than one for the three inputs. For input~$V_{rel}$, the cut granularity is larger when the velocity of the lead car is faster than the ego car. We guess that the model is more sensitive when the velocities of the lead and ego cars are similar.   


\section{Concluding Remarks}\label{sec.concluding.remarks}


In this paper, we considered the practical problem of using the equivalence class in testing autonomous vehicles. Starting with a set of categorizations, we realized that any combination is only an equivalence class ``candidate" and may be further subject to refinement. We formulated this idea using believed equivalence, established its connection with coverage criterion using combinatorial testing, and discussed analytical solutions for finding a refinement using optimization. As guided by the newly discovered test cases, the refinement offers traceability on why such a believed equivalence is formed. 

Currently, our initial evaluation is conducted using publicly available autonomous driving functions. We plan to integrate this technique in one of our research projects on automated valet parking to refine the process further and bring the concept into standardization bodies.

\vspace{2mm}
\noindent\textbf{(Acknowledgement)} This project has received funding from the European Union’s Horizon 2020 research and innovation programme under grant agreement No 956123.

\bibliographystyle{abbrv}

\end{document}